\documentclass[twoside,11pt]{article}
\pdfoutput=1

\usepackage{amsfonts}
\usepackage{amsmath}
\usepackage{amssymb}
\usepackage{mathtools}
\usepackage{url}
\usepackage{parskip}
\usepackage{amsthm}
\usepackage{mathabx}
\usepackage{array}
\usepackage{graphicx}
\usepackage{float}
\usepackage{xcolor}
\usepackage{xspace}
\usepackage[overload]{textcase}
\usepackage{algorithm}
\usepackage{algpseudocode}
\usepackage{hyperref}
\usepackage{cleveref}
 \hypersetup{hidelinks}

\DeclarePairedDelimiter{\inprod}{\langle}{\rangle}
\DeclarePairedDelimiter{\abs}{|}{|}

\DeclarePairedDelimiter{\ceil}{\lceil}{\rceil}
\DeclarePairedDelimiter{\floor}{\lfloor}{\rfloor}
\DeclarePairedDelimiter{\lr}{(}{)}

\newcommand{\lrb}[1]{\left(#1\right)}
\newcommand{\brb}[1]{\bigl(#1\bigr)}
\newcommand{\Brb}[1]{\Bigl(#1\Bigr)}
\newcommand{\bbrb}[1]{\biggl(#1\biggr)}
\newcommand{\Bbrb}[1]{\Biggl(#1\Biggr)}
\newcommand{\lsb}[1]{\left[#1\right]}
\newcommand{\bsb}[1]{\bigl[#1\bigr]}

\newcommand{\bbsb}[1]{\biggl[#1\biggr]}
\newcommand{\Bbsb}[1]{\Biggl[#1\Biggr]}
\newcommand{\lcb}[1]{\left\{#1\right\}}
\newcommand{\bcb}[1]{\bigl\{#1\bigr\}}

\newcommand{\bbcb}[1]{\biggl\{#1\biggr\}}
\newcommand{\lce}[1]{\left\lceil#1\right\rceil}
\newcommand{\bce}[1]{\bigl\lceil#1\bigr\rceil}

\newcommand{\ban}[1]{\bigl\langle#1\bigr\rangle}

\newcommand{\bban}[1]{\biggl\langle#1\biggr\rangle}

\newcommand{\E}{\mathbb{E}}
\newcommand{\I}{\mathbb{I}}
\newcommand{\R}{\mathbb{R}}
\newcommand{\pr}{\mathbb{P}}

\newcommand{\cA}{\mathcal{A}}
\newcommand{\cB}{\mathcal{B}}
\newcommand{\cC}{\mathcal{C}}

\newcommand{\cF}{\mathcal{F}}

\newcommand{\cI}{\mathcal{I}}
\newcommand{\cL}{\mathcal{L}}

\DeclareMathOperator*{\argmin}{arg\,min}

\newcommand{\loss}{\ell}
\newcommand{\eloss}{y}

\newcommand{\expset}{V}
\newcommand{\adv}{\theta}
\newcommand{\overbar}[1]{\mkern 1.5mu\overline{\mkern-1.5mu#1\mkern-1.5mu}\mkern 1.5mu}
\newcommand{\avgc}{\overbar{\cC}_T}
\newcommand{\inst}{\Xi}
\newcommand{\beainst}{\inst_{\mathrm{BEA}}}
\newcommand{\fginst}{\inst_{\mathrm{FG}}}
\newcommand{\instmap}{\rho}

\newcommand{\Chisqr}{{\chi^2}}
\newcommand{\chisqr}[2]{\Chisqr({#1}\,\|\,{#2})}
\newcommand{\bchisqr}[2]{\Chisqr\brb{{#1}\,\big\|\,{#2}}}

\renewcommand{\hat}{\widehat}
\newtheorem{theorem}{Theorem}[section]
\newtheorem{lemma}[theorem]{Lemma}

\usepackage{jair,theapa,rawfonts}

\jairheading{ }{ }{ }{ }{ }
\ShortHeadings{Improved Regret Bounds for Bandits with Expert Advice}
{Cesa-Bianchi, Eldowa, Esposito, \& Olkhovskaya}
\firstpageno{1}

\begin{document}

\title{Improved Regret Bounds for Bandits with Expert Advice}

\author{\name Nicolò Cesa-Bianchi \email nicolo.cesa-bianchi@unimi.it \\
    \addr Università degli Studi di Milano, Milan, Italy\\
    Politecnico di Milano, Milan, Italy
    \AND
    \name Khaled Eldowa \email khaled.eldowa@unimi.it\\
    \addr Università degli Studi di Milano, Milan, Italy \\
    Politecnico di Torino, Turin, Italy
    \AND
    \name Emmanuel Esposito \email emmanuel@emmanuelesposito.it\\
    \addr Università degli Studi di Milano, Milan, Italy\\
    Istituto Italiano di Tecnologia, Genoa, Italy
    \AND
    \name Julia Olkhovskaya \email julia.olkhovskaya@gmail.com\\
    \addr TU Delft, Delft, Netherlands
}

\maketitle

\begin{abstract}
 In this research note, we revisit the bandits with expert advice problem.
 Under a restricted feedback model, we prove a lower bound of order $\sqrt{K T \ln(N/K)}$ for the worst-case regret, where $K$ is the number of actions, $N>K$ the number of experts, and $T$ the time horizon.
 This matches a previously known upper bound of the same order and improves upon the best available lower bound of $\sqrt{K T (\ln N) / (\ln K)}$.
 For the standard feedback model, we prove a new instance-based upper bound that depends on the agreement between the experts and provides a logarithmic improvement compared to prior results.
\end{abstract}

\section{Introduction}
\label{sec:introduction}
The problem of bandits with expert advice provides a simple and general framework for incorporating contextual information into the non-stochastic multi-armed bandit problem. 
In this framework, the learner receives in every round a recommendation, in the form of a probability distribution over the actions, from each expert in a given set. 
This set of experts can be seen as a set of 
strategies
each mapping an unobserved context to a (randomized) action choice.
The goal of the learner is to minimize their expected regret with respect to the best expert in hindsight; that is, the difference between their expected cumulative loss and that of the best expert.
This problem was formulated by \citeA{auer2002}, who proposed the EXP4 algorithm as a solution strategy that has since become an important baseline or building block for addressing many related problems; for example, sleeping bandits \cite{kleinberg2010}, online multi-class classification \cite{daniely2013}, online non-parametric learning \cite{cesa-bianchi2017}, and non-stationary bandits \cite{luo2018}.
\citeA{auer2002} proved a bound of order $\sqrt{K T \ln N}$ on the expected regret incurred by the EXP4 strategy, where $T$ denotes the number of rounds, $K$ the number of actions, and $N$ the number of experts. 
This result is of a worst-case nature, in that it holds for any sequence of losses assigned to the actions and any sequence of expert recommendations. 

The appealing feature of the bound of \citeA{auer2002} is that it exhibits only a logarithmic dependence on the number of experts, in addition to the $\sqrt{K}$ dependence on the number of actions known to be unavoidable in the classical bandit problem, where the learner competes with the best fixed action. 
While the minimax regret\footnote{The best achievable worst-case regret guarantee.} in the latter problem has been shown to be of order $\sqrt{KT}$ \cite{audibert2009minimax}, a similar exact characterization remains missing for the expert advice problem.
\citeA{kale2014limited} studied a generalized version of the bandits with expert advice problem---originally proposed by \citeA{seldin2013}---where the learner is only allowed to query the advice of $M \le N$ experts.
When $M=N$, the results of \citeA{kale2014limited} imply an upper bound of order $\sqrt{\min\{K,N\} T (1+\ln(N/\min\{K,N\}))}$ on the minimax regret, improving upon the bound of \citeA{auer2002}. 
Unlike the latter, the logarithmic factor in \citeS{kale2014limited} bound diminishes as $K$ increases with respect to $N$, leading to a bound of order $\sqrt{NT}$ when $N \leq K$, which is tight in general as the experts in that case can be made to emulate an $N$-armed bandit problem.
This improved bound was achieved via the PolyINF algorithm \cite{audibert2009minimax,audibert2010} played on the expert set utilizing the importance-weighted loss estimators of EXP4. 
Later, \citeA{seldin2016} proved a lower bound of order $\sqrt{K T (\ln N) / (\ln K)}$ for $N \geq K$.

As these upper and lower bounds still do not match, the correct minimax rate remains unclear.
In this work, we take a step towards resolving this issue by showing that the upper bound is not improvable in general under a restricted feedback model in which the importance weighted loss estimators used by EXP4 or PolyINF remain implementable. 
In this restricted model, without observing the experts' recommendations, the learner picks an expert (possibly at random) at the beginning of each round, and the environment subsequently samples the action to be executed from the chosen expert's distribution.
Afterwards, the learner only observes the distributions of the experts that had assigned positive probability to the chosen action.
Via a reduction from the problem of multi-armed bandits with feedback graphs, we use the recent results of \citeA{chen2024} to obtain a lower bound of order $\sqrt{K T \ln(N/K)}$ for $N > K$.

Departing from the worst-case results discussed thus far, a few works have obtained instance-dependent bounds for this problem. 
The dependence on the instance can be in terms of the assigned sequence of losses through small loss bounds \cite<see>{allen-zhu2018}, or in terms of the sequence of expert recommendations through bounds that reflect the similarity between the recommended expert distributions (\citeR<see>[Theorem 18.3]{mcmahan2009,lattimore2020bandit}; \citeR{eldowa2024}).
Our focus here is on the latter case, where to the best of our knowledge the state of the art is a bound of order $\sqrt{\sum_t^T \cC_t \ln N}$, shown in the recent work of \citeA{eldowa2024} for the EXP4 algorithm. 
Here, $\cC_t$ is the (chi-squared) capacity of the recommended distributions at round $t$.
This quantity measures the dissimilarity between the experts' recommendations and satisfies $0 \leq \cC_t \leq \min\{K,N\}-1$.
Improving upon this result, we illustrate that it is possible to achieve a bound of order $\sqrt{\sum_t^T \cC_t \bigl(1+\ln(N / \max\{\avgc,1\})\bigr)}$, where $\avgc = \sum_t^T \cC_t / T$ is the average capacity.
This bound combines the best of the bound of \citeA{eldowa2024} (its dependence on the agreement between the experts) and that of \citeA{kale2014limited} (its improved log factor), simultaneously outperforming both.

\paragraph{Road map.}
We formalize the problem setting in the next section. 
In \Cref{sec:worst-case}, as a preliminary building block,
we present \Cref{alg:qFTRL}, an instance of the follow-the-regularized-leader (FTRL) algorithm with the (negative) $q$-Tsallis entropy as the regularizer. This algorithm is essentially equivalent to the PolyINF algorithm \cite<see>{audibert2011,abernethy2015}, which was used by \citeA{kale2014limited} to achieve the best known worst-case upper bound. 
We then show in \Cref{sec:instance-based} that combining this algorithm with a doubling trick allows us to achieve the improved instance-based bound mentioned above.
The lower bound for the restricted feedback setting is presented in \Cref{sec:lower-bound}. 
Finally, we provide some concluding remarks in \Cref{sec:conc}.

\section{Preliminaries}
\label{sec:setting}
\paragraph{Notation.}
For a positive integer $n$, $[n]$ denotes the set $\{1,\dots,n\}$.
For $x, y \in \R$, let $x \lor y \coloneqq \max\{x,y\}$ and $x \land y \coloneqq \min\{x,y\}$. 
Moreover, we define $x_+ \coloneqq x \lor 0$.
\paragraph{Problem setting.}
Let $\expset = [N]$ be a set of $N$ experts and $\cA = [K]$ be a set of $K$ actions. We consider a sequential decision-making problem where a learner interacts with an unknown environment for $T$ rounds. The environment is characterized by a fixed and unknown sequence of loss vectors $(\loss_t)_{t\in[T]}$, where $\loss_t \in [0,1]^K$ is the assignment of losses for the actions at round $t$, and a fixed and unknown sequence of expert advice $(\adv^i_t)_{i \in \expset,t \in [T]}$, where $\adv^i_t \in \Delta_K$ is the distribution over actions recommended by expert $i$ at round $t$.\footnote{For a positive integer $d$, we let $\Delta_d$ denote the probability simplex in $\R^d$ defined as $\{u \in \R^d \:\colon\: \sum_{j=1}^d u(j) = 1 \:\text{and}\: u(j) \geq 0 \: \forall j \in [d] \}$.} 
At the beginning of each round $t \in [T]$, the expert recommendations $(\adv_t^i)_{i \in \expset}$ are revealed to the learner, who then selects (possibly at random) an action $A_t \in \cA$ and subsequently suffers and observes the loss $\loss_t(A_t)$. 
For an expert $i \in V$, we define $\eloss_t(i) \coloneqq \sum_{a \in \cA} \adv_t^i(a) \loss_t(a)$ as its loss in round $t$.
The goal is to minimize the expected regret with respect to the best expert in hindsight:
\begin{equation*}
    R_T \coloneqq \E\bbsb{\sum_{t=1}^T \loss_t(A_t)} - \min_{i\in V} \sum_{t=1}^T \eloss_t(i) \,,
\end{equation*}
where the expectation is taken with respect to the randomization of the learner.

\section{\texorpdfstring{$q$}{q}-FTRL for Bandits with Expert Advice} \label{sec:worst-case}
The EXP4 algorithm can be seen as an instance of the FTRL framework  \cite<see, e.g.,>[Chapter 7]{orabona2023modern} where a distribution $p_t$ over the experts is maintained at each round $t$ and updated as follows
\[
    p_{t+1} \gets \argmin_{p \in \Delta_{N}} \eta\,\bban{\sum_{s=1}^{t} \hat{\eloss}_s, p} + \sum_{i \in \expset} p(i) \ln p(i) \,,
\]
where $\eta > 0$ is the learning rate, the second term is the negative Shannon entropy of $p$, and $\hat{\eloss}_s(i)$ is an importance-weighted estimate of $\eloss_s(i)$.
The action $A_t$ is drawn from the mixture distribution $\sum_{i \in \expset} p_{t}(i) \adv_t^i(\cdot)$.
Consider a more general algorithm (outlined in \Cref{alg:qFTRL}) where the negative Shannon entropy is replaced with the negative $q$-Tsallis entropy, which for $q \in (0,1)$ is given by
 \[
    \psi_q(x) \coloneqq \frac{1}{1-q} \left(1 - \sum_{i \in \expset} x(i)^q\right) \qquad \forall x \in \Delta_N \,.
\]
In the limit when $q \rightarrow 1$, the negative Shannon entropy is recovered.
The following theorem provides a regret bound for the algorithm.
This result is not novel, a similar bound is implied by Theorem~2 in \citeA{kale2014limited} for a closely related algorithm in a more general setting. 
We provide a concise proof of the result for completeness.
As mentioned before, when $N \leq K$, this bound is trivially tight in general. While when $N > K$, we prove an order-wise matching minimax lower bound in \Cref{sec:lower-bound} under additional restrictions on the received feedback.

\begin{algorithm}[t]
    \caption{$q$-FTRL for bandits with expert advice} \label{alg:qFTRL}
    \begin{algorithmic}
        \State \textbf{input:} $q \in (0,1)$, $\eta > 0$
        \State \textbf{initialization:} $p_1(i) \gets 1/N$ for all $i\in \expset$
        \For{$t = 1, \ldots, T$}
            \State receive expert advice $(\adv_t^i)_{i \in \expset}$
            \State draw expert $I_t \sim p_t$ and action $A_t \sim \adv_t^{I_t}$
            \State construct $\hat{\eloss}_{t} \in \R^N$ where $\hat{\eloss}_{t}(i) \coloneqq \frac{\adv_t^i(A_t)}{\sum_{j \in \expset} p_{t}(j) \adv_t^j(A_t)} \loss_t(A_t)$ for all $i\in \expset$            
            \State let $p_{t+1} \gets \argmin_{p \in \Delta_{N}} \eta\ban{\sum_{s=1}^{t} \hat{\eloss}_s, p} + \psi_q(p)$
        \EndFor
    \end{algorithmic}
\end{algorithm}

\begin{theorem}\label{thm:worst-case}
    \Cref{alg:qFTRL} run with
    \[
        q = \frac12\lrb{1 + \frac{\ln\brb{N/(K \land N)}}{\sqrt{\ln\brb{N/(K \land N)}^2+4} + 2}} \in [1/2,1)
        \:\: \text{ and } \:\:
        \eta = \sqrt{\frac{2q N^{1-q}}{T(1-q)(K \land N)^q}} \;,
    \]
    satisfies
    \[
        R_T \le 2\sqrt{e (K \land N) T \brb{2 + \ln\brb{N/(K \land N)}}} \;.
    \]
\end{theorem}
\begin{proof}
    Let $i^* \in \argmin_{i \in \expset} \sum_{t=1}^T \eloss_t(i)$, and note that $R_T = \E \sum_{t=1}^T \brb{\eloss_t(I_t) - \eloss_t(i^*)}$ as $\E\,\loss_t(A_t) = \E\,\eloss_t(I_t)$.
    For round $t \in [T]$, let $\cF_t \coloneqq \sigma(I_1,A_1,\dots,I_t,A_t)$ denote the $\sigma$-algebra generated by the random events up to the end of round $t$, and let $\E_t[\cdot] \coloneqq \E[\cdot \mid \mathcal{F}_{t-1}]$ with $\cF_{0}$ being the trivial $\sigma$-algebra. 
    For action $a \in \cA$, let $\phi_t(a) \coloneqq \sum_{i \in \expset} p_{t}(i) \adv_t^i(a)$ and note that conditioned on $\cF_{t-1}$, $A_t$ is distributed according to $\phi_t$.
    As $p_t$ is $\cF_{t-1}$-measurable, it is then easy to verify that $\E_t \hat{\eloss_t} = \eloss_t$.
    Hence, Lemma 2 in \citeA{eldowa2023} implies that 
    \begin{equation} \label{eq:tsallis-bias-variance}
        R_T \leq \frac{ N^{1-q}}{(1-q)\eta} + \frac{\eta}{2q} \sum_{t=1}^T \E \lsb{ \sum_{i \in V} p_t(i)^{2-q} \:\hat{\eloss}_t(i)^2 } \,.
    \end{equation}
    For fixed $t \in [T]$ and $i \in V$, we have that
    \begin{align} \label{eq:loss-second-moment-worst-case}
        \E_t \lsb{\hat{\eloss}_t(i)^2} &= \E_t \lsb{\frac{\adv_t^i(A_t)^2}{\phi_t(A_t)^2} \loss_t(A_t)^2 } \leq \E_t \lsb{ \frac{\adv_t^i(A_t)^2}{\phi_t(A_t)^2}} = \E_t \lsb{ \sum_{a \in \cA }\frac{\adv_t^i(a)^2}{\phi_t(a)^2} \I\{a=A_t\}} = \sum_{a \in \cA }\frac{\adv_t^i(a)^2}{\phi_t(a)} 
    \end{align}
    where the inequality holds because $\loss_t(A_t) \in [0,1]$ and the final equality holds because $\E_t\,\I\{a=A_t\} = \pr(a=A_t \mid \cF_{t-1}) = \phi_t(a)$. Hence, it holds that
    \begin{align*}
        \E_t \lsb{ \sum_{i \in V} p_t(i)^{2-q} \:\hat{\eloss}_t(i)^2 } 
        &=  \sum_{a \in \cA } \frac{ \sum_{i \in V} p_t(i)^{2-q}\adv_t^i(a)^2}{\phi_t(a)} \\
        &\leq \sum_{a \in \cA } \frac{ \sum_{i \in V} p_t(i)^{2-q}\adv_t^i(a)^{2-q}}{\phi_t(a)} \max_{i \in \expset} \adv_t^i(a)^q\\
        &\leq \sum_{a \in \cA } \frac{ \brb{\sum_{i \in V} p_t(i)\adv_t^i(a)}^{2-q}}{\phi_t(a)}  \max_{i \in \expset} \adv_t^i(a)^q\\ 
        &= \sum_{a \in \cA } \phi_t(a) \bbrb{\frac{\max_{i \in \expset} \adv_t^i(a)}{\phi_t(a)}}^{q} \leq \bbrb{\sum_{a \in \cA } \max_{i \in \expset} \adv_t^i(a)}^{q} \leq (K \land N)^{q}\,,
    \end{align*}
    where the second inequality follows from the superadditivity of $x^{2-q}$ for $x \geq 0$ and $q \in (0,1)$, the third inequality follows from the concavity of $x^q$ for $q \in (0,1)$ because of Jensen's inequality, and the last inequality holds since $\max_{i \in \expset} \adv_t^i(a) \leq \min\bigl\{1, \sum_{i\in \expset} \adv_t^i(a)\bigr\}$. Substituting back into \eqref{eq:tsallis-bias-variance} yields that
    \begin{equation*}
        R_T \leq \frac{ N^{1-q}}{(1-q)\eta} + \frac{\eta}{2q} (K \land N)^{q} T \,.
    \end{equation*}
    For brevity, let $\xi \coloneqq (K \land N)$. In a similar manner to the proof of Theorem 1 in \citeA{eldowa2023}, substituting the specified values of $\eta$ and $q$ allows us to conclude the proof:
    \begin{align*}
        R_T &\le  \sqrt{\frac{2N^{1-q}\xi^{q}}{q(1-q)}T} \\
        &= \sqrt{2T \exp\bbrb{1+\frac12 \ln\lr{\xi N} - \frac12 \sqrt{\ln\lrb{N/\xi}^2 + 4}} \bbrb{2 + \sqrt{\ln\lrb{N/\xi}^2+4}}} \\
        &\leq \sqrt{2T \exp\bbrb{1+\frac12 \ln\lr{\xi N} - \frac12 \ln\lrb{N/\xi}} \bbrb{2 + \sqrt{\ln\lrb{N/\xi}^2+4}}} \\
        &= \sqrt{2e\xi T \bbrb{2 + \sqrt{\ln\lrb{N/\xi}^2+4}}}
        \le 2\sqrt{e\xi T \sqrt{\ln\lrb{N/\xi}^2+4}} \\
        &\le 2\sqrt{e\xi T \lrb{2+\ln(N/\xi)}}
        \,. \qedhere
    \end{align*}
\end{proof}

\section{An Improved Instance-Based Regret Bound} \label{sec:instance-based}
We now obtain a more refined regret bound whose form is analogous to the bound of \Cref{thm:worst-case} except that it depends on the similarity between the experts' recommendations at each round, replacing $K \land N$ with an effective number of experts.
Before discussing the algorithm, we introduce some relevant quantities from \citeA{eldowa2024}. 
For any round $t \in [T]$ and $\tau \in \Delta_N$, define
\[
    Q_t(\tau) \coloneqq \sum_{i \in \expset} \tau(i) \bchisqr{\adv_t^i} {\textstyle{\sum_{j \in \expset}} \tau(j)\adv_t^j} = \sum_{a \in \cA} \frac{\sum_{i \in \expset} \tau(i)  \adv_t^i(a)^2}{\sum_{j \in \expset} \tau(j) \adv_t^j(a)} - 1 \,,
\]
where
$
        \chisqr{p}{q} \coloneqq 
        \sum_{a \in \cA} q(a) \brb{{p(a)}/{q(a)}-1}^2 =
        \sum_{a \in \cA} p(a)^2 / q(a) - 1
$
is the chi-squared divergence between distributions $p, q \in \Delta_K$. 
Additionally, let
\[
    \cC_t \coloneqq \sup_{\tau \in \Delta_N} Q_t(\tau) \qquad \text{and} \qquad \avgc \coloneqq \frac{1}{T} \sum_{t=1}^T \cC_t 
\]
be the chi-squared capacity of the recommended distributions at round $t$ and its average over the $T$ rounds. As remarked before,  $\cC_t$ is never larger than $(K \land N) - 1$ and can be arbitrarily smaller depending on the agreement between the experts at round $t$. 
In particular, it vanishes when all recommendations are identical.

The idea of \Cref{alg:q-ftrl-doubling} is to tune \Cref{alg:qFTRL} as done in \Cref{thm:worst-case} but with $\avgc$ replacing $K \land N$.
However, to avoid requiring prior knowledge of $\avgc$, we rely on a doubling trick to adapt to its value.
In a given round $t$, we maintain a running instance of \Cref{alg:qFTRL} tuned with an estimate for $\avgc$.
Let $m_t$ be the round when the present execution of \Cref{alg:qFTRL} had started.
If the current estimate is found to be smaller than $\frac{1}{2T} \sum_{s=m_t}^t Q_s(p_s)$, 
the algorithm is restarted and the estimate is (at least) doubled.
This quantity we test against is a simple lower bound for $\avgc/2$ that can be constructed without computing the capacity at any round.
As the value of $\avgc$ can be arbitrarily close to zero, the initial guess (which ideally should be a lower bound) is left as a user-specified parameter for the algorithm, and appears in the first (and more general) bound of \Cref{thm:instance-based}. 
The second statement of the theorem shows that choosing $\ln(e^2 N)/T$ as the initial guess suffices to obtain a bound of order $\sqrt{\sum_t^T \cC_t \bigl(1+\ln(N / \max\{\avgc,1\})\bigr)}$, up to an additive $\ln N$ term.
This simultaneously outperforms the $\sqrt{\sum_t^T \cC_t \ln N}$ bound of \citeA{eldowa2024} and the $\sqrt{(K \land N) T \bigl(1+\ln(N/(K \land N))\bigr)}$ bound of \citeA{kale2014limited}.

The proof combines elements from the proof of Theorem~1 of \citeA{eldowa2024} and the proof of Theorem~3 of \citeA{eldowa2023}, who adopt a similar algorithm to address online learning with time-varying feedback graphs.
Compared to the latter work,  
we require a more refined analysis to account for the case when $\avgc < 1$.
This refinement is achieved in part via the use of \Cref{lem:FTRL-Tsallis-bound}, which also allows adapting the analysis of \citeA{eldowa2024} to account for the fact that we use the $q$-Tsallis entropy as a regularizer in place of the Shannon entropy. 
\begin{algorithm} [t]
    \caption{$q$-FTRL with the doubling trick for bandits with expert advice }
    \label{alg:q-ftrl-doubling}
    \begin{algorithmic}[1]
        \State \textbf{input:} $J \in (0,N]$
        \State \textbf{initialization:} $r_1 \gets \bce{\log_2 J} - 1$, $m_1 \gets 1$, $p_1(i) \gets 1/N$ for all $i\in \expset$
        \State \textbf{define:} For each integer $r \in (-\infty, \log_2 N]$,
        \[
        q_r \coloneqq \frac12 \bbrb{1 + \frac{\ln(N/2^r)}{\sqrt{\ln(N/2^r)^2+4} + 2}}  
        \]
        \[
        \eta_r \coloneqq \min \lcb{ \sqrt{\frac{q_r (N^{1-q_r}-1)}{eT(1-q_r)\lrb{2^r}^{q_r}}}\,,\; \frac{q_r}{1-q_r} \Brb{1-e^{\frac{q_r-1}{2-q_r}}} }
        \] 
        \For{$t=1,\dotsc,T$} 
            \State receive expert advice $(\adv_t^i)_{i \in \expset}$
            \State draw expert $I_t \sim p_t$ and action $A_t \sim \adv_t^{I_t}$
            \State construct $\hat{\eloss}_{t} \in \R^N$ where $\hat{\eloss}_{t}(i) \coloneqq \frac{\adv_t^i(A_t)}{\sum_{j \in \expset} p_{t}(j) \adv_t^j(A_t)} \loss_t(A_t)$ for all $i\in \expset$
            \If{$\frac{1}{T} \sum_{s=m_t}^t Q_s(p_s) > 2^{r_t + 1}$}
                \State $p_{t+1}(i) \gets 1/N$ for all $i\in \expset$
                \State $r_{t+1} \gets \lce{\log_2 \brb{\frac{1}{T} \sum_{s=m_t}^t Q_s(p_s)}} - 1$,
                 $m_{t+1} \gets t+1$
            \Else
                \State $p_{t+1} \gets \argmin_{p \in \Delta_{N}} \eta_{r_t}\ban{\sum_{s=m_t}^{t} \hat{\eloss}_s, p} + \psi_{q_{r_t}}(p)$
                \State $r_{t+1} \gets r_t$,
                 $m_{t+1} \gets m_t$
            \EndIf
        \EndFor
    \end{algorithmic}
\end{algorithm}

\begin{theorem} \label{thm:instance-based}
    Assuming that $T \geq \ln(e^2N)$, \Cref{alg:q-ftrl-doubling} run with input $J \in (0,N]$ satisfies 
    \begin{multline*}
        R_T \leq 38 e \sqrt{\brb{\avgc \lor J} T \ln\lrb{\frac{e^2 N}{\avgc \lor J \lor 1}} } +  \log_2\lrb{\frac{\avgc}{J}}_+ \\+ \frac{18 e}{5} \log_2\lrb{\frac{4\lrb{\brb{JT \lor \avgc T} \land \ln(e^2N)}}{JT}}_+ \ln\brb{e^2 N} + 1 \,.
    \end{multline*}
    In particular, setting $J = \ln(e^2N)/T$ yields that
    \begin{align*}
        R_T \leq 38 e \sqrt{\avgc T \ln\lrb{\frac{e^2 N}{\avgc \lor 1}} } +  \log_2\lrb{\frac{\avgc T}{\ln(e^2N)}}_+ + 46 e \ln\brb{e^2 N} + 1 \,.
    \end{align*}
\end{theorem}
\begin{proof}
    For brevity, we define $U \coloneqq \avgc \lor J$.
    Let $s \coloneqq \bce{\log_2 J} - 1$ and $n \coloneqq \bce{\log_2 U} - 1$, the latter of which is the largest value that $r_t$ can take, since for any round $t$,
    \begin{equation*}
        \frac{1}{T} \sum_{s=m_t}^t Q_s(p_s) \leq \frac{1}{T} \sum_{s=1}^T Q_s(p_s) \leq \frac{1}{T} \sum_{s=1}^T \cC_s \leq 2^{n+1} \,.
    \end{equation*}
    Without loss of generality, we assume that for any  (integer) $r \in \{s,\dots,n\}$, there are at least two rounds in which $r_t=r$, and we use $T_r$ to refer to the index of the first such round. Additionally, we define $T_{n+1} \coloneqq T+2$. 
    Note that for any $r$ in this range, $q_r \in [1/2,1)$. 
    Let $i^* \in \argmin_{i \in \expset} \sum_{t=1}^T \eloss_t(i)$.
    We start by decomposing the regret over the intervals corresponding to fixed values of $r_t \in \{s,\dots,n\}$ and bounding the instantaneous regret at the last step of each but the last interval by $1$:
    \begin{align} 
        R_T &= \E \bbsb{ \sum_{t=1}^T \brb{\eloss_t(I_t) - \eloss_t(i^*)} } \nonumber\\
        &\leq  \E \bbsb{  \sum_{r=s}^n \sum_{t=T_r}^{T_{r+1}-2} \brb{\eloss_t(I_t) - \eloss_t(i^*)}} + n-s \nonumber\\ \label{eq:reg-dec}
        &\leq  \E \bbsb{ \sum_{r=s}^n \sum_{t=T_r}^{T_{r+1}-2} \brb{\eloss_t(I_t) - \eloss_t(i^*)}} + \log_2\brb{U / J} + 1 \,.
    \end{align}
    Let $\mathbf{e}_{i^*} \in \R^N$ be the indicator vector for $i^*$ and define $\tilde{\eloss}_t \in \R^N$ where $\tilde{\eloss}_t(i) \coloneqq \hat{\eloss}_t(i) - \loss_t(A_t)$ for every $i \in \expset$.
    Similar to the proof of Theorem 3 in \citeA{eldowa2023}, we note that for each $r \in \{s,\dots,n\}$, 
    \begin{align*}
        \E\Bbsb{\sum_{t=T_r}^{T_{r+1}-2} \brb{\eloss_t(I_t) - \eloss_t(i^*)}} 
        &= \E\Bbsb{\sum_{t=1}^T \I\bbcb{r_t = r, \frac{1}{T} \sum_{s=m_t}^t Q_s(p_s) \leq 2^{r_t}} \brb{\eloss_t(I_t) - \eloss_t(i^*)}} \nonumber\\
        &\stackrel{(a)}{=} \E\Bbsb{\sum_{t=1}^T \I\bbcb{r_t = r, \frac{1}{T} \sum_{s=m_t}^t Q_s(p_s) \leq 2^{r_t}} \inprod{p_t - \mathbf{e}_{i^*}, \hat{\eloss}_t}} \nonumber\\
        &\stackrel{(b)}{=} \E\Bbsb{\sum_{t=1}^T \I\bbcb{r_t = r, \frac{1}{T} \sum_{s=m_t}^t Q_s(p_s) \leq 2^{r_t}} \inprod{p_t - \mathbf{e}_{i^*}, \tilde{\eloss}_t}} \nonumber\\
        &= \E\Bbsb{\sum_{t=T_r}^{T_{r+1}-2} \inprod{p_t - \mathbf{e}_{i^*}, \tilde{\eloss}_t}} \nonumber\\
    \end{align*}
    where $(a)$ follows since 
    $\E_t\bsb{\eloss_t(I_t)} = \sum_{i \in V} p_t(i) \eloss_t(i)$, $\E_t\bsb{\hat{\eloss}_t} = \eloss_t$, and the indicator at round~$t$ is measurable with respect to $\cF_{t-1}$ (where $\cF_{t-1}$ and $\E_t$ are defined in the same way as in the proof of \Cref{thm:worst-case}); and
    $(b)$~follows since $p_t,\mathbf{e}_{i^*} \in \Delta_N$ and $\hat{\eloss}_t(i) - \tilde{\eloss}_t(i) = \loss_t(A_t)$ is identical for all $i \in \expset$.
    Similarly to the last argument, the fact that $\ban{\tilde{\eloss}_s - \hat{\eloss}_s, p - q} = 0$ holds for any $p,q \in \Delta_{N}$ at any round $s$ implies that $p_{t+1}$ can be equivalently defined as $\argmin_{p \in \Delta_{N}} \eta_{r_t}\ban{\sum_{s=m_t}^{t} \tilde{\eloss}_s, p} + \psi_{q_{r_t}}(p)$.
    Hence, using that $\tilde{\eloss}_t(i) \geq -1$, we can invoke \Cref{lem:FTRL-Tsallis-bound} (with $b=1$ and $c=e$) to obtain that
    \begin{align*}
        \sum_{t=T_r}^{T_{r+1}-2} \inprod{p_t - \mathbf{e}_{i^*}, \tilde{\eloss}_t} \leq \frac{ N^{1-q_r} - 1}{(1-q_r)\eta_r} + \frac{e\eta_r}{2q_r} \sum_{t=T_r}^{T_{r+1}-2}  \sum_{i \in \expset} p_t(i)^{2-q_r} \tilde{\eloss}_t(i)^2 \,.
    \end{align*}
    For any round $t \in [T]$ and action $a \in \cA$, recall the definition $\phi_t(a) \coloneqq \sum_{i \in \expset} p_{t}(i) \adv_t^i(a)$. Similar to \eqref{eq:loss-second-moment-worst-case} in the proof of \Cref{thm:worst-case}, we have that
    \begin{align*}
        \E_t \lsb{\tilde{\eloss}_t(i)^2} &= \E_t \lsb{\loss_t(A_t)^2 \frac{\brb{\adv_t^i(A_t) - \phi_t(A_t)}^2}{\phi_t(A_t)^2}}\\ 
        &\leq \E_t \lsb{\frac{\brb{\adv_t^i(A_t) - \phi_t(A_t)}^2}{\phi_t(A_t)^2}} \\
        &= \sum_{a \in \cA }\frac{\brb{\adv_t^i(a) - \phi_t(a)}^2}{\phi_t(a)} = \sum_{a \in \cA } \phi_t(a)\bbrb{\frac{\adv_t^i(a)}{\phi_t(a)} - 1}^2 = \chisqr{\adv_t^i}{\phi_t} \,.
    \end{align*}
    Hence, for any round $t$ and any $r \in \{s,\dots,n\}$, it holds that
    \begin{align*}
        \E_t \lsb{ \sum_{i \in \expset} p_t(i)^{2-q_r}  \tilde{\eloss}_t(i)^2 } 
        &\leq \sum_{i \in \expset} p_t(i)^{2-q_r}  \chisqr{\adv_t^i}{\phi_t} \\
        &= Q_t(p_t) \sum_{i \in \expset} \frac{p_t(i) \chisqr{\adv_t^i}{\phi_t}}{Q_t(p_t)} p_t(i)^{1-q_r} \\ 
        &\leq Q_t(p_t) \lrb{ \sum_{i \in \expset} \frac{p_t(i) \chisqr{\adv_t^i}{\phi_t}}{Q_t(p_t)} p_t(i)}^{1-q_r} \\ 
        &= Q_t(p_t)^{q_r} \lrb{ \sum_{i \in \expset} p_t(i)^2 \chisqr{\adv_t^i}{\phi_t}}^{1-q_r} \\
        &= Q_t(p_t)^{q_r} \lrb{ \sum_{i \in \expset} p_t(i)^2 \sum_{a \in \cA }\frac{\adv_t^i(a)^2}{\phi_t(a)} - \sum_{i \in \expset} p_t(i)^2}^{1-q_r} \\
        &= Q_t(p_t)^{q_r} \lrb{\sum_{a \in \cA } \frac{\sum_{i \in \expset} p_t(i)^2\adv_t^i(a)^2}{\sum_{j \in \expset} p_{t}(j) \adv_t^j(a)} - \sum_{i \in \expset} p_t(i)^2}^{1-q_r} \\
        &\leq Q_t(p_t)^{q_r} \lrb{\sum_{a \in \cA } {\sum_{i \in \expset} p_t(i) \adv_t^i(a)} - \sum_{i \in \expset} p_t(i)^2}^{1-q_r} \\
        &= Q_t(p_t)^{q_r} \lrb{1 - \sum_{i \in \expset} p_t(i)^2}^{1-q_r} \leq Q_t(p_t)^{q_r} \,,
    \end{align*}
    where the second inequality follows from the definition of $Q_t(p_t)$ and the fact that $x^{1-q_r}$ is concave in $x \geq 0$, and the third inequality uses the superadditivity of $x^2$ for non-negative real numbers and the non-negativity of the quantity in brackets. Let $T_{r:r+1} \coloneqq T_{r+1}-T_r-1$, it then holds that
    \begin{align*}
        \E \lsb{\sum_{t=T_r}^{T_{r+1}-2} \sum_{i \in \expset} p_t(i)^{2-q_r}  \tilde{\eloss}_t(i)^2 } &= \E \lsb{\sum_{t=1}^T \I\bbcb{r_t = r, \frac{1}{T} \sum_{s=m_t}^t Q_s(p_s) \leq 2^{r_t}} \sum_{i \in \expset} p_t(i)^{2-q_r}  \tilde{\eloss}_t(i)^2 } \\
        &\leq \E \lsb{ \sum_{t=T_r}^{T_{r+1}-2} Q_t(p_t)^{q_r} } \\
        &\leq \E\lsb{ T_{r:r+1} \lrb{\frac{1}{T_{r:r+1}} \sum_{t=T_r}^{T_{r+1}-2} Q_t(p_t)}^{q_r} } \\
        &\leq \E\lsb{ T_{r:r+1} \lrb{\frac{T}{T_{r:r+1}} 2^{r+1}}^{q_r} } \leq 2 T \lrb{2^r}^{q_r} \,,
    \end{align*}
    where the second inequality uses the concavity of $x^{q_r}$ in $x \geq 0$ and the third inequality uses that $(1/T) \sum_{t=T_r}^{T_{r+1}-2} Q_t(p_t) \leq 2^{r+1}$ since the algorithm is not reset in the interval $[T_r, T_{r+1}-2]$.
    Overall, we have shown that
    \begin{align*}
        \E\Bbsb{\sum_{t=T_r}^{T_{r+1}-2} \brb{\eloss_t(I_t) - \eloss_t(i^*)}} \leq \frac{ N^{1-q_r} - 1}{(1-q_r)\eta_r} + \frac{e\eta_r}{q_r} \lrb{2^r}^{q_r} T\,.
    \end{align*}
    If $\sqrt{\frac{q_r (N^{1-q_r}-1)}{eT(1-q_r)\lrb{2^r}^{q_r}}} \leq \frac{q_r}{1-q_r} \Brb{1-e^{\frac{q_r-1}{2-q_r}}}$,
    then substituting the values of $\eta_r$ and $q_r$ gives that
    \begin{align*}
        \frac{ N^{1-q_r} - 1}{(1-q_r)\eta_r} + \frac{e\eta_r}{q_r} \lrb{2^r}^{q_r} T &= 2\sqrt{\frac{e (N^{1-q_r}-1) \lrb{2^r}^{q_r}T }{q_r(1-q_r)} } \\
        &= 2\sqrt{\frac{N^{1-q_r}-1}{N^{1-q_r}}}\sqrt{\frac{e N^{1-q_r} \lrb{2^r}^{q_r} T}{q_r(1-q_r)} } \\
        &\leq 2e\sqrt{2} \sqrt{\frac{N^{1-q_r}-1}{N^{1-q_r}}} \sqrt{2^r \lrb{2+\ln(N 2^{-r})} T} \\
        &\leq 2e\sqrt{2} \lrb{\sqrt{\frac{\ln N}{\ln(N 2^{-r})}} \land 1} \sqrt{2^r \lrb{2+\ln(N 2^{-r})} T} \\
        &= 2e\sqrt{2} \sqrt{2^r \ln\brb{e^2 N ( 2^{-r} \land 1 )} T} \,,
    \end{align*}
    where the first inequality holds via the same arguments laid in the last passage of the proof of \Cref{thm:worst-case}, and the second inequality holds since
    \begin{align*}
        \frac{N^{1-q_r}-1}{N^{1-q_r}} 
        &=  1 - \exp\lrb{-\ln\brb{N^{1-q_r}}} \\
        &\leq (1-q_r)\ln N \\
        &= \frac12 \bbrb{1 - \frac{\ln(N/2^r)}{\sqrt{\ln(N/2^r)^2+4} + 2}} \ln N \\
        &= \frac{\ln N}{2\ln(N/2^r)} \lrb{2 + \ln(N/2^r) - \sqrt{\ln(N/2^r)^2+4}} \leq \frac{\ln N}{\ln(N/2^r)} \,,
    \end{align*}
    where the inequality follows from the fact that $1-e^{-x} \leq x$. Otherwise, if $\sqrt{\frac{q_r (N^{1-q_r}-1)}{e T(1-q_r)\lrb{2^r}^{q_r}}} > \frac{q_r}{1-q_r} \Brb{1-e^{\frac{q_r-1}{2-q_r}}}$, then $\eta_r$ takes the latter value and we obtain that
    \begin{align*}
        \frac{ N^{1-q_r} - 1}{(1-q_r)\eta_r} + \frac{e\eta_r}{q_r} \lrb{2^r}^{q_r} T &\leq  \frac{ N^{1-q_r} - 1}{(1-q_r)\eta_r} + \eta_r \frac{ N^{1-q_r} - 1}{(1-q_r)} \lrb{\frac{1-q_r}{q_r \Brb{1-e^{\frac{q_r-1}{2-q_r}}}}}^2
        \\&= 2\frac{ N^{1-q_r} - 1}{q_r \Brb{1-e^{\frac{q_r-1}{2-q_r}}}} \\
        &\leq \frac{18 \brb{N^{1-q_r} - 1}}{5 q_r(1-q_r)} \\
        &= \frac{18 \lrb{2^r}^{-q_r}\brb{N^{1-q_r} - 1}\lrb{2^r}^{q_r}}{5q_r(1-q_r)} \\
        &\leq \frac{18 e}{5} \lrb{2^r}^{1-q_r} \ln\brb{e^2 N ( 2^{-r} \land 1 )} \\
        &\leq \frac{18 e}{5} \brb{ 1 \lor \sqrt{2^r} } \ln\brb{e^2 N ( 2^{-r} \land 1 )} \,,
    \end{align*}
    where the last inequality holds since $q_r \geq 1/2$, and the second inequality holds since 
    \begin{align*}
        1-e^{\frac{q_r-1}{2-q_r}} &\geq \frac{1-q_r}{2-q_r} - \frac{1}{2} \lrb{\frac{1-q_r}{2-q_r}}^2 
        = \frac{3 - q_r}{2(2-q_r)^2} (1-q_r) 
        \geq \frac{5}{9} (1-q_r) \ln\brb{e^2 N ( 2^{-r} \land 1 )}\,,
    \end{align*}
    where the first step uses that $e^{-x} \leq 1 - x + x^2/2$ for $x \geq 0$, and the last step uses again that $q_r \geq 1/2$. Hence, the results above yield that
    \begin{multline} \label{eq:reg-interval}
        \E\Bbsb{\sum_{t=T_r}^{T_{r+1}-2} \brb{\eloss_t(I_t) - \eloss_t(i^*)}} \leq  \max \bbcb{ 2e\sqrt{2} \sqrt{2^r T \ln\brb{e^2 N ( 2^{-r} \land 1 )}},\\
        \frac{18 e}{5} \brb{ 1 \lor \sqrt{2^r} } \ln\brb{e^2 N ( 2^{-r} \land 1 )}  } \,.
    \end{multline}
    Let $M \coloneqq \ln(e^2N)/T$ and $m \coloneqq \log_2 M$, and note that $m \leq 0$ (and $M\leq 1$) by the assumption that $T \geq \ln(e^2N)$.
    In the case when $n \leq 0$, we have that
    \begin{align*}
        &\E \lsb{ \sum_{r=s}^n \sum_{t=T_r}^{T_{r+1}-2} \brb{\eloss_t(I_t) - \eloss_t(i^*)}} \\
        &\hspace{7em}\leq \frac{18 e}{5} \brb{(n \land \floor{m}) - s + 1}_+ \ln\brb{e^2 N} + 2e\sqrt{2} \sum_{r=n \land \ceil{m}}^n \sqrt{2^r T \ln\brb{e^2 N}}\\
        &\hspace{7em}\leq \frac{18 e}{5} \log_2\brb{{4(U \land M)/J}}_+ \ln\brb{e^2 N} + 8e \sqrt{2 U T \ln\brb{e^2 N}} \,,
    \end{align*}
    where the second inequality uses that
    \begin{align*}
        \sum_{r={\alpha}}^n \brb{\sqrt{2}}^r = \brb{\sqrt{2}}^{\alpha} \sum_{r=0}^{n-{\alpha}} \brb{\sqrt{2}}^r = \brb{\sqrt{2}}^{\alpha} \frac{\brb{\sqrt{2}}^{n-{\alpha}+1}-1}{\sqrt{2} - 1} \leq \frac{\sqrt{2}}{\sqrt{2} - 1} \brb{\sqrt{2}}^{n} \leq 4 \sqrt{U} \,,
    \end{align*}
    with $\alpha \coloneqq n \land \ceil{m}$.
    Otherwise, if $n > 0$, then
    \begin{align*}
        &\E \lsb{ \sum_{r=s}^n \sum_{t=T_r}^{T_{r+1}-2} \brb{\eloss_t(I_t) - \eloss_t(i^*)}} \\
        &\quad\leq \frac{18 e}{5} \log_2\lrb{4M/J}_+ \ln\brb{e^2 N} + 8e \sqrt{2 T \ln\brb{e^2 N}} + \E \lsb{ \sum_{r=s_+}^n \sum_{t=T_r}^{T_{r+1}-2} \brb{\eloss_t(I_t) - \eloss_t(i^*)}} \\
        &\quad\leq \frac{18 e}{5} \log_2\lrb{4M/J}_+ \ln\brb{e^2 N} + 8e \sqrt{2 T \ln\brb{e^2 N}} + \frac{18 e}{5} \sum_{r=0}^n \sqrt{2^r \ln\brb{e^2 N 2^{-r}} T} \\
        &\quad\leq \frac{18 e}{5} \log_2\lrb{4M/J}_+ \ln\brb{e^2 N} + 8e \sqrt{2 T \ln\brb{e^2 N}} + 26 e \sqrt{U T \ln\brb{e^2 N / U} } \\
        &\quad\leq \frac{18 e}{5} \log_2\lrb{4M/J}_+ \ln\brb{e^2 N} +  38 e \sqrt{U T \ln\brb{e^2 N / U} } \,,
    \end{align*}
    where the first inequality follows from the analysis of the first case with $n=0$, the second inequality uses that $r\geq0$ and the assumption that $T \geq \ln(e^2N)$, the third inequality uses Lemma 4 in \citeA{eldowa2023}, and the fourth uses that $x \ln(e^2N/x)$ is increasing in $[0,eN]$ and that $U \geq 2$ in this case. The theorem then follows by combining the bounds provided for the two cases with \eqref{eq:reg-dec}.
\end{proof}

\section{A Lower Bound for Restricted Advice via Feedback Graphs} \label{sec:lower-bound}

In this section, we provide a novel lower bound on the minimax regret for a slightly harder formulation of the multi-armed bandit problem with expert advice.
We consider a setting where the learner picks an expert $I_t$ (possibly at random) at the beginning of each round $t \in [T]$ without observing any of the experts' recommendations beforehand.
The action $A_t$ to be executed is subsequently drawn by the environment from the chosen expert's distribution, i.e., $A_t \sim \adv_t^{I_t}$. 
Afterwards, the learner observes $A_t$, the incurred loss $\ell_t(A_t)$, and the advice $\adv_t^i$ only of experts $i \in V$ that have the drawn action $A_t$ in their support, i.e., $\adv_t^i(A_t)>0$.
For experts outside this set, the learner can only infer that, by definition, $\adv_t^i(A_t)=0$. 
We will refer to this variation of the problem as the multi-armed bandit with \emph{restricted} expert advice (note that this differs from the limited expert advice model studied by \citeR{kale2014limited}).
Observe that \Cref{alg:qFTRL} is still implementable in this scenario and guarantees a regret upper bound of order $\sqrt{\xi T \lrb{1+\ln(N/\xi)}}$ for $\xi \coloneqq K \land N$, as previously analyzed.
Here we show that the regret of \Cref{alg:qFTRL} is the best regret we can hope for, up to constant factors, for any number $K$ of actions and any number $N$ of experts.
While a $\Omega(\sqrt{NT})$ regret lower bound in the case $N \le K$ is immediate (as mentioned before), the following theorem provides an $\Omega\brb{\sqrt{KT\ln(N/K)}}$ lower bound when $N > K$, improving upon the $\Omega\brb{\sqrt{KT(\ln N)/(\ln K)}}$ lower bound of \citeA{seldin2016}.

In what follows, we fix $N > K \geq 2$. We derive the lower bound relying on a reduction from the multi-armed bandit problem with feedback graphs \shortcite{mannor2011side,alon2013bandits,alon2015beyond,alon2017journal}.
In this variant of the bandit problem, we assume there exists a graph $G=(\expset,E)$ over a finite set $\expset=[N]$ of actions from which the learner selects one action $J_t \in \expset$ at each round $t \in [T]$.
Then, the learner observes the losses of the neighbours of $J_t$ in $G$.
For the construction of the lower bound, it suffices to assume that $G$ is undirected and contains all self-loops, i.e., $(i,i) \in E$ for each $i \in V$.
Consequently, the learner always observes the loss of the selected action and the graph $G$ is strongly observable---%
see \citeA{alon2015beyond} for a classification of feedback graphs.
We particularly focus on a specific family of graphs (also considered in the recent work of \citeR{chen2024}) where the $N$ vertices are partitioned into disjoint cliques with self-loops. 
Precisely, we let $M \coloneqq \floor{K/2} \ge 1$ be the number of disjoint cliques in $G$. 
For any $k \in [M]$, let $C_k$ be the set of vertices of the $k$-th clique in $G$.
Since each $C_k$ is a clique with all self-loops, we have that $(i,j) \in E$ if and only if $i,j \in C_k$ for some $k \in [M]$, and thus $E = \bigcup_{k\in [M]} (C_k \times C_k)$.
Additionally, for our purposes, we only consider the partition into cliques $C_k = \bcb{i \in [N] : i \equiv k \mod M}$ of roughly the same size $\abs{C_k} \ge \floor{N/M} \ge \floor{2N/K} \ge N/K$.

Hence, we will focus on the class of instances, denoted by $\fginst$, of the multi-armed bandit problem with feedback graphs where the graph assumes the particular structure described above.
In particular, any instance $\cI \in \fginst$ is defined as a tuple $\cI \coloneqq (T,G,\cL)$ containing the number $T$ of rounds, the feedback graph $G=(V,E)$ over $V=[N]$ composed of the disjoint cliques $C_1, \dots, C_M$ as defined above, and the sequence $\cL \coloneqq (\loss_t)_{t\in [T]}$ of binary loss functions $\loss_t\colon V\to \{0,1\}$ over $V$.
On the other hand, we let $\beainst$ be the class of instances for the multi-armed bandit problem with restricted expert advice, with $N$ experts and $K$ actions.
An instance $\cI \in \beainst$ is a tuple $\cI \coloneqq \brb{T,V,\cA,\Theta,\cL}$ containing the number $T$ of rounds, the set $V=[N]$ of experts, the set $\cA = [K]$ of actions, the sequence $\Theta \coloneqq (\adv_t^i)_{i\in V,t\in [T]}$ of expert advice where $\adv_t^i \in \Delta_K$, and the sequence $\cL \coloneqq (\loss_t)_{t\in [T]}$ of loss functions $\loss_t\colon \cA\to \{0,1\}$ over $\cA$.
The sought result is established by showing that the worst-case regret of any algorithm against a particular subset of instances in $\beainst$ is order-wise at least as large as the minimax regret on $\fginst$, combined with a lower bound on the latter quantity by \citeA{chen2024}.

\begin{theorem} \label{thm:lower-bound}
    Let $\cB$ be any possibly randomized algorithm for the multi-armed bandit problem with restricted expert advice for any number $K\ge 2$ of actions $\cA = [K]$ and any number $N>K$ of experts $V = [N]$.
    Then, for a sufficiently large $T$, there exist a sequence $\loss_1,\dots,\loss_T\colon \cA \to \{0,1\}$ of binary loss functions and a sequence $(\adv^i_t)_{i \in \expset,t \in [T]}$ of expert advice such that 
    the expected regret of $\cB$ is $\Omega\brb{\sqrt{KT\ln(N/K)}}$.
\end{theorem}
\begin{proof}
    We first describe a reduction from the multi-armed bandit problem with feedback graphs to the multi-armed bandit problem with restricted expert advice.
    We accomplish this by providing a mapping $\instmap\colon\fginst\to\beainst$ from the considered instance class $\fginst$ of the former problem to the instance class $\beainst$ of the latter.
    
    Consider any instance $\cI \coloneqq (T,G,\cL) \in \fginst$ and recall that $G = (V,E)$ is a union of $M = \floor{K/2}$ disjoint cliques $C_1,\dots,C_M$ over $V = [N]$.
    The mapped instance $\instmap(\cI) \coloneqq (T,V,\cA,\Theta,\cL') \in \beainst$ is defined over the same number of rounds $T$ and an experts set corresponding to the actions $V$ in the original instance $\cI$, whose sequence of recommendations is provided by $\Theta = (\adv_t^i)_{i\in V,t\in [T]}$.
    We first observe that the cardinality of the new action set $\cA = [K]$ does relate to the number of cliques $M$.
    In particular, considering the partition of experts given by the cliques in $G$, we also partition the actions (in the expert advice instance $\rho(\cI)$) by associating $2$ actions to each clique.
    Precisely, for any $k \in [M]$, we
    associate
    actions $\cA_k \coloneqq \{2k-1,2k\}$ to $C_k$.
    If $K$ is even, this partitions the entire set of actions $\cA$, while it leaves out action $K$ otherwise.
    We can ignore the latter case and assume $K$ is even without loss of generality, since we can 
    otherwise leave action $K$ outside of the support of any expert advice $\adv_t^i \in \Delta_K$ in the following construction (thus becoming a spurious action).

    Second, we focus on the construction of the loss sequence $\cL' \coloneqq (\loss_1',\dots,\loss_T')$.
    For any $t \in [T]$, we define $\loss_t' \in \{0,1\}^\cA$ as
    \begin{align*}
        \loss_t'(2k-1) \coloneqq 0
        \qquad \text{and} \qquad
        \loss_t'(2k) \coloneqq 1
        \qquad \forall k \in [M] \cdot
    \end{align*}
    Finally, we define the sequence of expert advice $(\adv_t^i)_{i\in V,t\in [T]}$ depending on the sequence of losses $\cL$ of the starting instance $\cI$.
    For any $t \in [T]$, any $k \in [M]$, and any $i \in C_k$, we define $\adv_t^i \in \Delta_K$ as
    \begin{align*}
        \adv_t^i \coloneqq \begin{cases}
            \delta_{2k-1} & \text{if $\loss_t(i) = 0$}\\
            \delta_{2k} & \text{if $\loss_t(i) = 1$}
        \end{cases} \;,
    \end{align*}
    where $\delta_j \in \Delta_K$ is the Dirac delta at $j \in \cA$.
    This ensures that the loss of expert $i$ at round $t$, given by $\eloss_t(i) = \sum_{a \in \cA} \adv_t^i(a) \loss_t'(a)$ coincides with $\ell_t(i)$, the loss of action $i$ in the original feedback graphs instance at the same round.
    Moreover, the knowledge of $\ell_t(i)$ suffices to infer $\adv_t^i$.  

    At this point, given our instance mapping $\instmap$ and our algorithm $\cB$, we design an algorithm $\cB_\instmap$ for 
    the class $\fginst$.
    Consider any instance $\cI \in \fginst$.
    Over the interaction period, the algorithm $\cB_\instmap$, without requiring prior knowledge of $\cI$, maintains a running realization of $\cB$ on instance $\instmap(\cI)$.
    At any round $t \in [T]$, let $I_t$ be the expert selected by algorithm $\cB$ in $\instmap(\cI)$, and let $k_t \in [M]$ be the index of the clique $I_t$ belongs to, i.e., $I_t \in C_{k_t}$.
    Algorithm $\cB_\instmap$, interacting with the instance $\cI$, executes action $J_t = I_t$ provided by $\cB$ and observes the losses $(\ell_t(i))_{i \in C_{k_t}}$.
    Then, thanks to the design of the mapping $\instmap$,
    $\cB_\instmap$ can construct and provide $\cB$ the feedback it requires and which complies with instance $\instmap(\cI)$.
    Namely, it determines that $A_t = 2k_t - 1$ if $\ell_t(J_t)=0$ or else that $A_t = 2k_t$, then passes $A_t$, its loss $\ell'_t(A_t)$ (trivially determined), and the restricted advice $(\adv_t^i)_{i \in C_{k_t}}$ to $\cB$.
    The last of which is a super-set of the recommended distributions having positive support on $A_t$ since $A_t$ is never picked by experts outside $C_{k_t}$ by construction.  
    
    Now, let
    \[
        R^{\cB}(\cI') \coloneqq \E\left[\sum_{t=1}^T \loss_t'(A_t)\right] - \min_{i\in V} \sum_{t=1}^T \sum_{a\in \cA} \adv_t^i(a)\loss_t'(a) = \E\left[\sum_{t=1}^T \eloss_t(I_t)\right] - \min_{i\in V} \sum_{t=1}^T \sum_{a\in \cA} \adv_t^i(a)\loss_t'(a)
    \]
    be the expected regret of algorithm $\cB$ on some instance $\cI' = \brb{T,V,\cA,(\adv_t^i)_{i\in V, t\in [T]}, (\loss_t')_{t\in [T]}} \in \beainst$.
    Similarly, let
    \[
        R^{\cB_\instmap}(\cI) \coloneqq \E\left[\sum_{t=1}^T \loss_t(J_t)\right] - \min_{i\in V} \sum_{t=1}^T \loss_t(i)
    \]
    be the expected regret of algorithm $\cB_\instmap$ on some instance $\cI = \brb{T,G,(\loss_t)_{t\in [T]}} \in \fginst$.
    Since $J_t = I_t$, we have that $\eloss_t(I_t) = \loss_t(J_t)$ via the properties of $\instmap$ laid out before.
    Hence, we can conclude that $R^{\cB}(\instmap(\cI)) = R^{\cB_\instmap}(\cI)$ for any instance $\cI \in \fginst$.
    Define $\instmap(\fginst) \coloneqq \bcb{\instmap(\cI) : \cI \in \fginst} \subseteq \beainst$ as the subclass of instances in $\beainst$ obtained from $\fginst$ via $\instmap$.
    Then, it holds that 
    \begin{align*}
        \sup_{\cI\in\beainst} R^\cB(\cI)
        \ge \sup_{\cI\in\instmap(\fginst)} R^\cB(\cI)
        = \sup_{\cI\in\fginst} R^\cB(\instmap(\cI))
        = \sup_{\cI\in\fginst} R^{\cB_\instmap}(\cI) \;.
    \end{align*}
    On the other hand, Lemma~E.1 in \citeA{chen2024} implies that 
    \begin{align*}
        \sup_{\cI\in\fginst} R_T^{\cB_\instmap}(\cI) = \Omega\Bbrb{\sqrt{T \sum_{k \in [M]} \ln(1+\abs{C_k})}} = \Omega\brb{\sqrt{KT\ln(N/K)}}
    \end{align*}
    for sufficiently large $T$ since $\sum_{k \in [M]} \ln(1+\abs{C_k}) \ge M\ln(N/M) \ge K\ln(2N/K)/4$, thus concluding the proof.
\end{proof}

\section{Conclusion} \label{sec:conc}
As the lower bound of \Cref{thm:lower-bound} was proved for a harder formulation of the problem, it remains to be shown whether the same impossibility result holds for the more standard setup.
We conjecture it should be possible to prove such a lower bound. If it indeed holds, this would imply that the minimax regret in the two variants is of the same order; that is, as far as we are only concerned with the worst-case regret, the standard feedback setup would be shown to be essentially as hard as the restricted one.

\acks{%
NCB, EE, and KE acknowledge the financial support from the MUR PRIN grant 2022EKNE5K (Learning in Markets and Society), the FAIR (Future Artificial Intelligence Research) project, and the EU Horizon RIA under grant agreement 101120237, project ELIAS (European Lighthouse of AI for Sustainability).
}

\appendix
\section{Auxiliary Results}
\begin{lemma} \label{lem:FTRL-Tsallis-bound}
    Let $q\in(0,1)$, $b > 0$, $c > 1$, and $(y_t)_{t=1}^T$ be a sequence of non-negative loss vectors in $\R^N$ satisfying $y_t(i) \geq -b$ for all $t \in [T]$ and $i \in [N]$. Let $(p_t)_{t=1}^{T+1}$ be the predictions of FTRL with decision set $\Delta_N$ and the $q$-Tsallis regularizer~$\psi_{q}$ over this sequence of losses; that is, $p_1 = \argmin_{p \in \Delta_{N}}  \psi_q(p)$, and for $t \in [T]$,
    \[
     p_{t+1} = \argmin_{p \in \Delta_{N}} \eta \sum_{s=1}^{t} \ban{ y_s, p} + \psi_q(p) \,,
    \]
    assuming the learning rate $\eta$ satisfies $0 < \eta \leq \frac{q}{(1-q)b}\Brb{1-c^{\frac{q-1}{2-q}}}$.
    Then for any $u \in \Delta_{N}$,
    \begin{equation*}
        \sum_{t=1}^T \langle p_t - u, y_t \rangle \leq \frac{ N^{1-q} - 1}{(1-q)\eta} + \frac{\eta c}{2q} \sum_{t=1}^T  \sum_{i=1}^N p_t(i)^{2-q} \:y_t(i)^2 \,.
    \end{equation*}
\end{lemma}
\begin{proof}
    Let $p'_{t+1} \coloneqq \argmin_{p \in \R^N_{\geq0}} \langle p, y_t \rangle + D_{\psi_q}(p,p_t)$, where $D_{\psi_q}(\cdot,\cdot)$ denotes the Bregman divergence based on $\psi_q$.
    Via Lemma~7.14 in \cite{orabona2023modern} we have that
    \begin{align*}
         \sum_{t=1}^T \langle p_t - u, y_t \rangle &\leq \frac{\psi_q(u)-\psi_q(p_1)}{\eta} + \frac{\eta}{2 q} \sum_{t=1}^T \sum_{i=1}^N z_t(i)^{2-q} \:y_t(i)^2 \\
         &\leq \frac{ N^{1-q} - 1 }{(1-q)\eta} + \frac{\eta}{2 q} \sum_{t=1}^T \sum_{i=1}^N z_t(i)^{2-q} \:y_t(i)^2 \,,
    \end{align*}
    where $z_t$ lies on the line segment between $p_t$ and $p'_{t+1}$. A simple derivation shows that
    \begin{align*}
        p'_{t+1}(i) = p_t(i) \lrb{\frac{1}{1 + \eta \frac{1-q}{q} y_t(i) p_t(i)^{1-q} }}^{\frac{1}{1-q}} \,,
    \end{align*}
    for each $i \in [N]$. On the other hand, it holds that
    \begin{align*}
        \eta \frac{1-q}{q} y_t(i) p_t(i)^{1-q} \geq - \eta \frac{1-q}{q} b p_t(i)^{1-q} \geq - \eta \frac{1-q}{q} b \geq c^{\frac{q-1}{2-q}} - 1 \,,
    \end{align*}
    where the first inequality uses that $y_t(i) \geq -b$ (and that $p_t(i),\eta > 0$), the second uses that $p_t(i) \leq 1$, and the third uses that $\eta \leq \frac{q}{(1-q)b}\Brb{1-c^{\frac{q-1}{2-q}}}$. This entails that $p'_{t+1}(i) \leq c^{\frac{1}{2-q}} p_t(i)$, which implies that $z_t(i) \leq c^{\frac{1}{2-q}} p_t(i)$ concluding the proof.
\end{proof}

\vskip 0.2in
\bibliography{sample}
\bibliographystyle{theapa}

\end{document}